\documentclass{article}

\usepackage[final]{neurips_2020}

\usepackage[utf8]{inputenc} 
\usepackage[T1]{fontenc}    
\usepackage{hyperref}       
\usepackage{url}            
\usepackage{booktabs}       
\usepackage{amsfonts}       
\usepackage{nicefrac}       
\usepackage{microtype}      
\usepackage{dsfont}
\usepackage{graphicx}
\usepackage{subcaption}
\usepackage{adjustbox}

\usepackage{natbib}
\usepackage{soul}
\usepackage{footmisc}
\usepackage{url}
\usepackage{wrapfig}
\usepackage{amsmath}
\usepackage{subcaption}
\usepackage{booktabs}
\usepackage{multirow}
\usepackage{graphicx}
\usepackage{algorithm}
\usepackage{algorithmic}
\usepackage{amsthm}
\usepackage{amssymb}
\usepackage{multicol}
\usepackage{amsmath}
\usepackage{tabularx}
\usepackage{adjustbox}
\usepackage{times}
\usepackage{bbm,amssymb,amsmath,amsfonts,graphicx,ifthen,twoopt,algorithm,algorithmic,color}

\newcommand{\bs}[1]{\boldsymbol{#1}}

\author{Anonymous
}
\newtheorem{thm}{Theorem}
\newtheorem{lem}{Lemma}

\newtheorem{pro}{Property}
\newtheorem{definition}{Definition}

\newtheorem{rmk}{Remark}

\title{Double-Linear Thompson Sampling for Context-Attentive Bandits}

\author{
 Djallel Bouneffouf$^1$, Raphaël Féraud$^2$, Sohini Upadhyay$^1$, Yasaman Khazaeni$^1$ and Irina Rish$^3$ \\
  IBM Research AI\\
  \texttt{$^1$\{firstname.lastname\}@ibm.com} \\
  \texttt{$^2$\{raphael.feraud@orange.com} \\
  \texttt{$^3$\{irina.rish@mila.quebec} \\
}
\begin{document}
\maketitle
\begin{abstract}
In this paper, we analyze and  extend an online learning framework known as   {\em Context-Attentive Bandit},  motivated by various practical applications, from medical diagnosis to dialog systems,  where due to  observation costs only a small subset of a potentially large number of  context variables can be  observed at each iteration; however,  the agent has a freedom to  choose which variables to observe.  We derive a novel algorithm, called {\em Context-Attentive Thompson Sampling  (CATS)}, which builds upon the Linear Thompson Sampling approach,  adapting it to {\em Context-Attentive Bandit}  setting.  We provide a theoretical regret analysis and an extensive empirical evaluation demonstrating  advantages of the proposed approach over several baseline methods  on a variety of real-life datasets.
\end{abstract}

\section {Introduction}

The {\em contextual bandit} problem is a variant of the extensively studied multi-armed bandit problem ~\cite{LR85,28,UCB,lin2018}, where at each iteration,   the agent observes an $N$-dimensional {\em context} ({\em feature vector}) and uses it,  along with the rewards of the arms played in the past, to decide which arm to play   \cite{langford2008epoch,29,27,BouneffoufF16}.
The objective of the agent is to learn the 
relationship between the context and reward, in order to find the best arm-selection policy for maximizing cumulative reward  over time.

 Recently, a promising variant  of contextual bandits, called a {\em Context Attentive Bandit (CAB)} was  proposed in \cite{BouneffoufRCF17}, where no context is given by default, but the agent can request  to  observe (to focus its "attention" on) a limited number of context variables at each iteration.  
We propose here an extension of this problem setting:  a small subset of $V$ context variables is revealed  at each iteration (i.e. partially observable context), followed by the agent's choice of  additional $U$  features, where   $V$,  $U$, and the set of $V$ immediately observed features are  fixed at all iterations. 
The agent must learn to select both the best  additional features  and, subsequently,   the best arm to play, given the resulting $V+U$ observed features. (The original {\em Context Attentive Bandit} corresponds to $V=0$.)

The proposed setting is motivated by several real-life applications.
For instance, in a clinical setting, a doctor may first take a look at patient's medical record (partially observed context)  to decide which medical test (additional context variables) to perform, before choosing a treatment plan (selecting an arm to play). 
It  is often  too costly or even impossible to conduct all possible tests (i.e., observe the full context); therefore, given the limit on the number of tests,  the doctor must decide which subset of tests will result into maximally effective treatment choice (maximize the reward).
Similar problems can arise in multi-skill orchestration for AI agents. For example, in dialog orchestration, a user's query is first directed to a number of domain-specific agents, each providing a different response,   and then the best response is selected. However, it might be too costly to request the answers from {\em all} domain-specific experts, especially in  multi-purpose dialog systems with a very large number of  domains experts. Given a limit on the number of experts to use for each query, the orchestration agent must choose the best subset of experts to use.  In this application, the query is the immediately observed part of the overall context, while the responses of domain-specific experts are the initially unobserved features from which a limited subset must be selected and observed, before choosing an arm, i.e. deciding on  the best out of the available responses. For multi-purpose dialog systems, such as, for example, personal home assistants, retrieving features or responses from every domain-specific agent is computationally expensive or intractable, with the potential to cause a poor user experience, again underscoring the need for effective feature selection. 

Overall, the main contributions of this paper include: (1) a generalization of {\em Context Attentive Bandit}, and a first lower bound for this problem,
(2) an algorithm called {\em Context Attentive Thompson Sampling} for stationary and non-stationary environments, and its regret bound in the case of stationary environment,  and  (3) an extensive empirical evaluation demonstrating advantages of our proposed algorithm over the previous context-attentive bandit approach \cite{BouneffoufRCF17}, on  a range of datasets in both stationary and non-stationary settings.


\section{Related Work}
\label{sec:related}

The contextual bandit (CB) problem has been extensively studied in the past, and a variety of solutions have been proposed. In LINUCB ~\cite{13,abbasi2011improved,LiCLW11,survey2019}, Neural Bandit \cite{AllesiardoFB14} and in linear Thompson Sampling ~\cite{AgrawalG13,BalakrishnanBMR19,BalakrishnanBMR19j}, a linear dependency is assumed between the expected reward given the context and an action taken after observing this context; the representation space is modeled using a set of linear predictors. However, the context is assumed to be {\em fully observable}, which is not the case in this work. 
Motivated by dimensionality reduction tasks, \cite{YadkoriPS12} studied a sparse variant of stochastic linear bandits, where only a relatively small and unknown subset of features is relevant to a multivariate function optimization. It presents an application to the problem of optimizing a function that depends on many features, where only a small, initially unknown subset of features is relevant.
Similarly, \cite{CarpentierM12} also considered high-dimensional stochastic linear bandits with sparsity. There the authors combined  ideas from compressed sensing and bandit theory to derive a novel algorithm. In \cite{oswal2019linear}, authors explores a new form of the linear bandit problem in which the algorithm receives the usual stochastic rewards as well as stochastic feedback about which features are relevant to the rewards and propose an algorithm that can achieve $ O( \sqrt{T})$ regret, without prior knowledge of which features are relevant.
In \cite{bastani2015online}, the problem is formulated as a multi-arm bandit (MAB) problem with high-dimensional covariates, and a new efficient bandit algorithm based on the LASSO estimator is presented. However, {\em the above work, unlike ours, assumes  fully observable  context variables}, which is not always the case in some  applications, as discussed in the previous section.
In \cite{BouneffoufRCF17} the authors proposed the novel framework of {\em contextual bandit with restricted context}, where observing the  whole feature vector at each iteration is too costly or impossible for some reasons; however, the  agent can  request to observe the values of an arbitrary subset of features within a given budget, i.e. the limit on the number of features observed.  This paper explores a more general problem,  and unlike \cite{BouneffoufRCF17}, we provide a theoretical analysis of the proposed problem and the proposed algorithm.

The {\em Context Attentive Bandit} problem is related to the {\em budgeted learning} problem, where a learner can access only a limited number of attributes from the training set or from the test set (see for instance \cite {CSO2011}). In \cite {FKK2016}, the authors studied the {\em online budgeted learning} problem. They showed a significant negative result: for any $\delta > 0$ no algorithm can achieve regret bounded by $O(T^{1-\delta})$ in polynomial time.  For overcoming this negative result, an additional assumption is necessary. Here, following \cite {DG2014}, we assume that the expected reward of selecting a subset of features is the sum of the expected rewards of selecting individually the features. We obtain an efficient algorithm, which has a linear algorithmic complexity in terms of time horizon.

 \section{Context Attentive Bandit Problem}
\label{sec:statement}
We now introduce the problem setting, outlined in Algorithm \ref{problem}. 
Let $C$ be a set of $N$ features.
At each time point $t$ the environment generates a feature vector $\mathbf{c}(t)=(c_1(t),...,c_N(t),1) \in \mathbf{\mathds{R}}^{N+1}$,
which the agent cannot  observe fully, but a partial observation of the context is allowed: the values of a subset of $V<N$ {\em observed} features $C^V \subset C$, are revealed: $\bs{c}^V=(c^V_1(t),...,c^V_N(t),1) \in \mathds{R}^{N+1}$,   $\forall i \in C, c^V_i(t)=c_i(t)\mathds{1}{\{i \in C^V\}}$. 
Based on this partially observed context, the agent is allowed to request an additional subset of $U$ {\em unobserved} features $C^U \subset C \setminus C^V$, $V+U \leq N$.
The goal of the agent is  to maximize its total reward over time via (1) the optimal choice of the additional set of features $C^U$, given the initial observed features $\mathbf{c}^V(t)$, and (2) the optimal choice of an   arm $k \in [K]=\{1,...,K\}$ based on $\mathbf{c}^{V+U}(t)=(c^{V+U}_1(t),...,c^{V+U}_N(t),1) \in \mathds{R}^{N+1}$, $\forall i \in C,  c^{V+U}_i(t)=c_i(t)\mathds{1}{\{i \in C^V \lor i \in C^U\}}$. We assume $P_r(r | \bs{c},k)$, an unknown probability distribution of the reward given the context and the action taken in that context. In the following the expectations are taken over the probability distribution $P_r(r | \bs{c},k)$. 
\begin{algorithm}[tbh]
	\caption{Context Attentive Bandit Problem (CAB)}
	\label{alg:CBP}
	\begin{algorithmic}[1]
		\FOR{$t:=1$ to $T$}
		\STATE {\bfseries } Context $\bs{c}(t)$ is chosen by the environment
		\STATE {\bfseries } The values $\bs{c}^V(t)$ of a subset $C^V \subset C$ are revealed 
 	\STATE {\bfseries } The agent selects  a subset  $C^U \subseteq C \setminus C^V$
\STATE{\bfseries} The values $\bs{c}^U(t)$ are revealed;  
\STATE{\bfseries} The agent chooses an arm $k(t):= \pi(\bs{c}^{V+U}(t)) $
				\STATE {\bfseries } The reward $r_{k(t)}$ is sampled from distribution $P_r(r|\bs{c},k)$ and it is revealed
		\STATE {\bfseries } The agent updates the policy $\pi \in \Pi_{C^{V+U}}$
			\STATE {\bfseries } 
			\ENDFOR
	\end{algorithmic}
	\label{problem}
\end{algorithm}

\noindent{\bf The contextual bandit problem.}
Following \cite{langford2008epoch}, this problem is defined as follows.
At each time point $t \in \{1,...,T\}$, an agent is presented with a {\em context} (i.e. {\em feature vector}) $\textbf{c}(t) \in \mathbf{\mathds{R}}^{N+1}$
  before choosing an arm $k  \in [K]$.
Let ${\bf r(t)} = (r_{1}(t),...,$ $r_{K}(t)) \in \mathds{R}^K$ denote a reward vector, where $r_k(t)$ is the reward at time $t$  associated with the arm $k$.
Let $\pi: \mathbf{\mathds{R}}^{N+1} \rightarrow [K]$ denote a policy,
mapping a context  $\bs{c}(t) \in \bs{\mathds{R}}^{N+1}$ into an arm $k \in [K]$.   We assume that the expected reward is a linear function of the context.

\textbf{Assumption 1 (linear contextual bandit):} \textit{Whatever the subset of selected features $C^U \subset C \setminus C^V$ the expected reward is a linear function of the context:
$\mathds{E}[r_k | \bs{c}^{V+U}(t)]= \textbf{c}^{V+U}(t)^{\top} \bs{\mu}_k$, where $\bs{\mu}_k \in \mathds{R}^{N+1}$ is an unknown parameter, $\forall k \in [K]$ $||\bs{\mu}_k|| \leq 1$, and $||\bs{c}^{V+U}(t)|| \leq 1$.}

\noindent{\bf Contextual Combinatorial Bandit.}
The contextual combinatorial bandit problem \cite{qin2014contextual} can be viewed as a game where the agent sequentially observes a context $\bs{c}^V(t)$, selects a subset $C^{U}(t)\subset C \setminus \{C^V\}$ and observes the reward corresponding to the selected subset. The goal is to maximize the reward over time.
Let $r_{U}|\textbf{c}^V(t),\pi \in \mathds{R}$  be the reward associated with the set of selected features $C^U$ knowing the context vector $\mathbf{c}^{U+V}(t)$ and the policy $\pi$.
We have 
$r_{U}|\textbf{c}^V(t),\pi = r_{k(t)}|\textbf{c}^{V+U}(t)$, where 
$k(t)=\pi(\bs{c}^{V+U}(t))$.
Each feature $i \in C^U$ is associated with the corresponding random variable $r_{i}|\textbf{c}^V(t),\pi\in \mathds{R}$ 
which indicates the reward obtained when choosing the $i$-th feature at time $t$.

\textbf{Assumption 2 (linear contextual combinatorial bandit):} 
\textit{
the mean reward of selecting the set of features 
$C^U \subset C \setminus \{C^V\}$ is:  $\mathds{E}[r_{U}|\textbf{c}^V(t),\pi]=\sum_{i\in C^U} \mathds{E}[r_{i}|\textbf{c}^V(t),\pi]$, 
 and the expectation of the reward of selecting the feature $i$ is a linear function of the context vector $\textbf{c}^V(t)$: $\mathds{E}[r_{i}|\textbf{c}^V(t),\pi] =  \textbf{c}^V(t)^{\top}  \bs{\theta}_i$,
where $\bs {\theta}_i\in \mathds {R}^{N+1}$ is an unknown weight vector associated with the feature $i$, $\forall i \in C \setminus C^V$ $||\bs{\theta}_i|| \leq 1$, and $||\bs{c}^{V}(t)|| \leq 1$.
}
Let $\Pi_{C^{V+U}}$ be the set of linear policies such that only the features coming from $C^{V+U}$ are used, where $C^V$ is a fixed subset of $C$, and
$C^U$ is any subset of $C \setminus C^V$.
The objective of {\it Contextual Attentive Bandit} (Algorithm \ref {alg:CBP}) is to find an optimal  policy $\pi^* \in \Pi_{C^{V+U}}$, over $T$ iterations or time points,  so that the total reward is maximized.  
\begin{definition}[Optimal Policy for CAB]
The optimal policy $\pi^*$ for handling the CAB problem is selecting the arm at time $t$:
$k^*(t) = \arg\max_{k \in [K], C^U \in C \setminus C^V} \bs{c}^{U+V}(t)^\top \bs{\mu}_k= \arg\max_{k \in [K]}  \bs{c}^{V+U^*}(t)^\top \bs{\mu}_k$,
where $C^{U^*}=\arg\max_{C^U \subset C \setminus C^V} r_{U}| \bs{c}^V(t),\bs{\mu}_1,...,\bs{\mu}_K=  \arg\max_{C^U \subset C \setminus C^V} \sum_{i \in C^U} \bs{c}^V(t)^\top \bs{\theta}_i$.
\end{definition} 

\begin{definition}[Cumulative regret] \label {regret}
The cumulative regret over $T$ iterations of the policy $\pi \in \Pi_{C^{V+U}}$, is defined as $R(T) = \sum ^{T}_{t=1} \mathds{E} [r_{\pi^*(\mathbf{c}(t))}] - \sum^{T}_{t=1} \mathds{E} [r_{\pi(\mathbf{c}(t))}]$.
\end{definition}

\begin{pro}[Regret decomposition] \label {regret_dec} The cumulative regret over $T$ iterations of the policy $\pi \in \Pi_{C^{V+U}}$ can be rewritten as following:
\begin{align*}
R(T) = & \sum ^{T}_{t=1} \mathds{E} [r_{\pi^*(\mathbf{c}(t))}] - \sum^{T}_{t=1} \mathds{E} [r_{\pi(\mathbf{c}(t))}] = \sum_{t=1}^T\left[
\bs{c}^{V+U^*}(t)^\top \bs{\mu}_{k^*(t)}-\bs{c}^{V+U}(t)^\top \bs{\mu}_{k(t)}\right] \nonumber 
\end{align*}

\begin{align*}
= & \sum_{t=1}^T\left[
\bs{c}^{V+U^*}(t)^\top \bs{\mu}_{k^*(t)}-\bs{c}^{V+U^*}(t)^\top \bs{\mu}_{k(t)}\right]  
 + \sum_{t=1}^T\left[
\bs{c}^{V+U^*}(t)^\top \bs{\mu}_{k(t)}-\bs{c}^{V+U}(t)^\top \bs{\mu}_{k(t)}\right] \nonumber 
\end{align*}
\begin{align*}
= & \sum_{t=1}^T\left[ \bs{c}^{V+U^*}(t)^\top \bs{\mu}_{k^*(t)}-\bs{c}^{V+U^*}(t)^\top \bs{\mu}_{k(t)}\right] + \sum_{t=1}^T \left[ 
\mathds{E}[r_{U^*}|\bs{c}^V(t),\pi] -
\mathds{E}[r_{U(t)}|\bs{c}^V(t),\pi] \right] 
\end{align*}
\begin{align*}
= & \sum_{t=1}^T\left[ \bs{c}^{V+U^*}(t)^\top \bs{\mu}_{k^*(t)}-\bs{c}^{V+U^*}(t)^\top \bs{\mu}_{k(t)}\right]  
 + \sum_{t=1}^T\left[   \sum_{i \in C^{U^*}} \textbf{c}^V(t)^{\top} \bs{\theta}_i -  \sum_{i \in C^{U(t)}} \textbf{c}^V(t)^{\top} \bs{\theta}_i \right]. 
\end{align*}
\end {pro}

\begin{rmk}\label{remark11}
CAB problem generalizes {\em contextual bandit with restricted context} problem (\cite{BouneffoufRCF17}). Indeed, when the subset of observed context $C^V$ is empty, the reward of selecting the feature $i$ is given by $\theta_{i,N+1}$, which is the coordinate $N+1$ of the vector $\bs{\theta}_i$.
\end{rmk}

Before introducing an algorithm for solving the above CAB problem, we will derive a lower bound on the expected regret of any algorithm used to solve this problem.

\begin{thm}\label{theorem1}
For any policy $\pi \in \Pi_{C^{V+U}}$  solving Context Attentive Bandit problem (Algorithm 1) under Assumption 1.1 and 1.2, there exists probability distribution $P_c(c)$ and $P_{r}(c,k)$, such that
the lower bound of the regret accumulated by $\pi$ over $T$ iterations is :
\begin{eqnarray*}
 \Omega \left( \sqrt {(U+V)T}  + U\sqrt {VT}\right).
\end{eqnarray*}
\end{thm}

\begin {proof}
The left term of the regret (see Property \ref {regret_dec}) is lower bounded by the lower bound of linear bandits in dimension $U+V$ (Theorem 2 in \cite{ChuLRS11}), while
the right term is lower bounded by the lower bound of $U$ linear bandits in dimension $V$.
\end {proof}

\section{Context Attentive Thompson Sampling (CATS) }
\begin{algorithm}[ht]
 \caption{Context Attentive Thompson Sampling (CATS)}
\label{alg:CATSO}
\begin{algorithmic}[1]
 \STATE {\bfseries }\textbf{Require:} 
 $N$, $V$, $U$, $K$, $T$, $C^V$, $\alpha >0$, $\lambda(t)$
 \STATE {\bfseries }\textbf{Initialize:} $\forall k \in [K], A_k:=I_{N+1}$, $\bs{g}_k := \bs{0}_{N+1}$, $\hat{\bs{\mu}_k}:= \bs{0}_{N+1}$, and $\forall i \in [N]$, $B_i:=I_{N+1}$, $\bs{z}_i:=\bs{0}_{N+1}$, $\hat{\bs{\theta}_i}:=\bs{0}_{N+1}$.
 \STATE {\bfseries }\textbf{Foreach} $t= 1, 2, . . . ,T$ \textbf{do}
 \STATE \quad observe  $\bs{c}^V(t)$  
 \STATE {\bfseries } \quad \textbf{Foreach} context feature $i \in C \setminus C^V$ \textbf{do}
 \STATE \quad \quad  Sample $\Tilde{\bs{\theta}}_i$  from $\mathcal{N}(\hat{\theta}_i, \alpha^2 B_i^{-1})$
\STATE {\bfseries } \quad\textbf{End do}
\STATE \quad Sort $(\bs{c}^V(t)^\top \Tilde{\bs{\theta}}_1,... \bs{c}^V(t))^\top \Tilde{\bs{\theta}}_N)$ in decreasing order
\STATE \quad Select $C^{U}(t):=\{i \in C \setminus C^V, \bs{c}^V(t)^\top \Tilde{\bs{\theta}}_i \geq \bs{c}^V(t)^\top \Tilde{\bs{\theta}}_V \}$
 \STATE  {\bfseries }\quad
observe values $\bs{c}^{V+U}(t)$ 

 \STATE {\bfseries }\quad\textbf{Foreach} arm $k= 1,...,K$ \textbf{do} 
 \STATE \quad\quad Sample $\Tilde{\bs{\mu}}_k$ from $\mathcal{N}(\hat{\bs{\mu}}_k, \alpha^2 A_k^{-1})$ distribution
\STATE \quad \textbf{End do}
 \STATE {\bfseries }\quad Select arm $k(t):= \arg\max_{k\subset [K] }  \bs{c}^{V+U}(t)^\top \Tilde{\bs{\mu}}_k$
 \STATE {\bfseries }\quad Observe $r_{k}(t)$
\STATE \quad$A_k:= A_{k}+ \bs{c}^{V+U}(t)\bs{c}^{V+U}(t)^\top $, \quad $\bs{g}_k := \bs{g}_k + \bs{c}^{V+U}(t)r_{k}(t)$, \quad $\hat{\bs{\mu}}_k := A_k^{-1} \bs{g}_k$
\STATE \quad \textbf{Foreach} $i \in C^U$
\STATE \quad\quad$B_i:= \lambda(t) B_{i}+ \bs{c}^V(t)\bs{c}^V(t)^{\top} $, \quad $\bs{z}_i: = \bs{z}_i + \bs{c}^V(t)r_{k}(t)$,\quad$\hat{\bs{\theta}}_i := \lambda(t) B_i^{-1} \bs{z}_i$
 \STATE {\bfseries }\quad \textbf{End do}
 \STATE {\bfseries }\textbf{End do}
\end{algorithmic}
\end{algorithm}

We now propose an algorithm for solving the CAB problem, called {\em Context-Attentive Thompson Sampling (CATS)}, and summarize it in Algorithm \ref{alg:CATSO}.
The basic idea of CATS is to use linear Thompson Sampling \cite {AgrawalG13} for solving the $U$ linear bandit problems for selecting the set of additional relevant features $C^{U^*}$ knowing $\bs{c}^V(t) \in C^V$, and for selecting the best arm knowing $\bs{c}^{V+U^*}(t) \in C^{V+U^*}$. Linear Thompson Sampling assumes a Gaussian prior for the likelihood function, which corresponds in CAB for arm $k$ to $\Tilde{\bs{\mu}}_k \sim \mathcal{N}(\bs{c}^{V+U}(t)^\top \hat{\bs {\mu}}_k, \alpha^2)$, and for feature $i$ to, $\Tilde{\bs{\theta}}_i \sim \mathcal{N}(\bs{c}^{V}(t)^\top \hat{\bs {\theta}}_i, \alpha^2)$. Then the posterior at time $t+1$ are respectively $P(\Tilde {\bs{\mu}}_k|r_k(t)) \sim \mathcal {N}(\hat{\bs{\mu}}_k,\alpha^2A^-1_k)$ for arm $k$, and $P(\Tilde {\bs{\mu}}_k|r_k(t)) \sim \mathcal {N}(\hat{\bs{\theta}}_i,\alpha^2B^-1_k)$ for feature $i$.

The algorithm takes the total number of features $N$, the number of features initially observed $V$, the number of additional features to observe $U$, the set of observed features $C^V$, the number of actions $K$, the time horizon $T$, the  distribution parameter $\alpha >0$ used in linear Thompson Sampling, and a function of time $\lambda(t)$, which is used for adapting the algorithm to non-stationary linear bandits.

At each iteration $t$ the values $\bs{c}^V(t)$ of features in the subset $C^V$ are observed (line 4 Algorithm \ref {alg:CATSO}). Then the vector parameters $\Tilde{\bs{\theta}}_i$ are sampled for each feature $i \in C^V$ (lines 5-7) from the posterior distribution (line 6). Then the subset of best estimated features at time $t$ is selected (lines 8-9).
Once the feature vector $\bs{c}^{U(t)}$ is observed line 10, Linear Thompson Sampling is applied in steps 11-15 to choose an arm. When the reward of selected arm is observed (line 15) the parameters are updated lines 16-19.

\begin{rmk} [Algorithmic complexity]\label{remark2} At each time set, Algorithm \ref {alg:CATSO} sorts a set of size $V$ and inverts $U+1$ matrices in dimensions $N+1$ that leads to an  algorithmic complexity  in $O (V \log V +(U+1)(N+1)^2) T$.
\end{rmk}

Due to assumption 2, CATS algorithm benefits from a linear algorithmic complexity, overcoming the negative result stated in \cite {FKK2016}. Before providing  an upper bound of the regret of CATS in $\Tilde {O} (\sqrt {T}$) ($\Tilde O$ hides logarithmic factors), we need an additional assumption on the noise $\eta_k(t)$.

\textbf{Assumption 3 (Sub-Gaussian noise):} \textit{ $\forall\ C^{U+V} \in C$ and $\forall k \in [K]$, the noise $n_{k}(t) = r_{k(t)}|\textbf{c}^{V+U}(t) - \bs{c}^{U+V}(t)^{\top} \hat {\mu_k}$ is conditionally
 $\rho$-sub-Gaussian with $\rho \geq 0$, that is for all $t \geq 1$},
\begin{equation*}
\forall\; \lambda \in \mathds{R}, \quad E[e^{\lambda n_k(t)}] \leq \exp\left(\frac{\lambda^2\rho^2}{2}\right) \text{ .}
\end{equation*}

\begin{lem}\label{TH2} (Theorem 2 in \cite{AgrawalG13})  When the measurement noise $n_k(t)$ satisfies Assumption 2, $\forall t \in \{1,...,T\} \quad || \bs{c}_k(t) \leq 1 ||$, $||\bs{\mu} \leq 1||$, $\alpha = \rho \sqrt {9N\log \frac{T}{\delta}}$, and $\epsilon = \frac{1}{\log T}$ the regret $R(T)$
of Thompson Sampling in the Linear bandit problem with $K$ parameters is upper bounded with a probability $1-\delta$ by: 

\begin{equation*}
\label{eq:bound2}
O\left(N\sqrt{K T \log K} (\log T)^{3/2} \log \frac{1}{\delta}\right), \text {  where $0 < \delta < 1$.}
\end{equation*}
\end{lem}

We can now derive the following result.

\begin{thm}\label{theorem2}
 When the measurement noise $\eta_k(t)$ satisfies Assumption 2, $\forall t \in \{1,...,T\} \quad || \bs{c}_k(t) \leq 1 ||$, $||\bs{\mu} \leq 1||$, $\alpha = \rho \sqrt {9N\log \frac{T}{\delta}}$, $\lambda(t)=1$ and $\epsilon = \frac{1}{\log T}$ the regret $R(T)$
of CATS (Algorithm 2) is upper bounded with a probability $1-\delta$ by:
\begin{align*} 
 \label{eq:bound2}
O \left( \left((U+V) \sqrt { KT \log K } +UV \sqrt { (N-V) T \log (N-V) } \right)(\log T)^{3/2}\log \frac {1}{\delta}\right).
\end{align*}
\end{thm}

\begin{proof}
For upper bounding the left term of the regret (see Property \ref {regret_dec}), we apply Lemma \ref {TH2}, and for upper bounding the right term, which is the regret of Thompson Sampling in $U$ linear bandit problems in $V$ dimensions with respectively $N-V,N-V-1,...,N-V-U-1$ parameters, we apply Lemma \ref {TH2}.
\end {proof}

 Theorem 2 states that the regret of CATS depends on the following two terms: the left term is the regret due to selecting a sub-optimal arm, while the right term is the regret of selecting a sub-optimal subset of features. We can see that there is still a gap between the lower bound of the {\it Context Attentive Bandit} problem and the upper bound of the proposed algorithm. 
The left term of the lower bound scales in $\Omega(\sqrt{(U+V)T)}$, while the left term of the upper bound of CATS scales in $\Tilde{O}((U+V)\sqrt{KT\log K})$, where $\Tilde{O}$ hides logarithmic factor.
The right term of the lower bound scales in $\Omega(U\sqrt{VT})$, while the right term of the upper bound of CATS scales in $\Tilde{O}((U+V)\sqrt{(N-V)T\log (N-V)})$. These gaps are due to the upper bound of regret of CATS, which uses Lemma \ref{TH2}. This suggests that the use of linear bandits based on an upper confidence balls, which scale in $\Tilde {O}\sqrt{dT}$ \cite {abbasi2011improved} ($d$ is the dimension of contexts), could reduce this theoretical gap. As we show in the next section, we choose the Thompson Sampling approach for its better empirical performances.

\section{Experiments}
We compare the proposed CATS algorithm with:
(1) Random-EI: in addition to the $V$ observed features, this algorithm selects a Random subset of features of the specified size $U$ at each Iteration
(thus, Random-EI), and then invokes the linear Thompson sampling algorithm.
(2) Random-fix: this algorithm invokes linear Thompson sampling on a subset of $U+V$ features, where the subset $V$ is randomly selected once prior to seeing any data samples, and remains fixed.
(3) The state-of-art method for context-attentive bandits proposed in \cite{BouneffoufRCF17},  Thompson Sampling with Restricted Context (TSRC):  TSRC solves the CBRC (contextual bandit with restricted context) problem  discussed earlier: at each iteration, the algorithm  decides on $U+V$ features to observe (referred to as {\em unobserved context}). In our setting, however, $V$ out of $N$ features are immediately observed at each iteration (referred to as {\em known context}), then TSRC decision mechanism is used to select $U$ additional unknown features to observe, followed by linear Thompson sampling on U+V features. (4) CALINUCB: where we replace the contextual TS in CATS with LINUCB. (5) CATS-fix: is heuristic where we stop the features exploration after some iterations $T'= 10\%, 20\% .... 90\%$ (we report here the an average over the best results). Empirical evaluation of Random-fix, Random-EI, CATS-fix, CALINUCB,  CATS \footnote{Note that we have used the same exploration parameter value used in \cite{chapelle2011empirical} for TS and LINUCB type algorithms which are $TS \in \{0.25, 0.5, 1\}$ and $LINUCB \in \{ 0.51, 1, 2\}$} and TSRC was performed on several publicly available datasets, as well as  on a proprietary corporate dialog orchestration dataset. 
Publicly available Covertype and CNAE-9 were featured in the original TSRC paper and Warfarin \cite{sharabiani2015revisiting} is a historically popular dataset for evaluating bandit methods.

To simulate the known and unknown context space, we randomly fix $10\%$ of the context feature space of each dataset to be known at the onset and explore a subset of $U$ unknown features. To consider the possibility of nonstationarity in the unknown context space over time, we introduce a weight decay parameter $\lambda(t)$ that reduces the effect of past examples when updating the CATS parameters. We refer to the stationary case as CATS and fix $\lambda(t) = 1$. For the nonstationary setting, we simulate nonstationarity in the unknown feature space by duplicating each dataset, randomly fixing the known context in the same manner as above, and shuffling the unknown feature set - label pairs. Then we stochastically replace events in the original dataset with their shuffled counterparts, with the probability of replacement increasing uniformly with each additional event. We refer to the nonstationary case as NCATS and use $\lambda(t)$ as defined by the GP-UCB algorithm \cite{srinivas2009gaussian}. We compare NCATS to NCATS-fix and NCALINUCB which are the non stationary version of CATS-fix and CALINUCB. we have also compare NCATS to the Weighted TSRC (WTSRC), the nonstationary version of TSRC also developed by \cite{BouneffoufRCF17}. WTSRC makes updates to its feature selection model based only on recent events, where recent events are defined by a time period, or "window" $w$. We choose $w=100$ for WTSRC. We report the total average reward divided by T over 200 trials across a range of $U$ corresponding to various percentages of $N$ for each algorithm in Table \ref{tab:uci}.

\begin{table*}[ht]
\centering
\caption{Total average reward, $O=10\%$}
\label{tab:uci}
\begin{subtable}[ht]{0.4\linewidth}\centering
\caption{Stationary setting}
\label{table:stationary}
\begin{adjustbox}{width=\linewidth}
\begin{tabular}{c|l|l|l}  
 & \multicolumn{3}{c}{\textbf{\textit{Warfarin}}}\\
\hline
{\textit{U} }      & 20\%                        & 40\%                 & 60\%               \\ \hline
 {TSRC}         &      53.28 $\pm$ 1.08      &   57.60 $\pm$ 1.16            &            59.87 $\pm$ 0.69           \\ \hline
{CATS}          &      53.65 $\pm$ 1.21      &   \textbf{ 58.55 $\pm$ 0.67}  &   \textbf{ 60.40 $\pm$ 0.74}     \\ \hline
{CATS-fix}   & \textbf{53.99 $\pm$ 1.02}  &   58.67 $\pm$ 0.65            &          60.07 $\pm$ 0.54    \\ \hline
{CALINUCB}      &         52.17 $\pm$ 0.89   &   57.23 $\pm$ 0.53            &          60.29 $\pm$ 0.66    \\ \hline
{Random-fix}    & 51.05 $\pm$ 1.31 & 53.55 $\pm$ 0.97 & 55.15 $\pm$ 0.83         \\ \hline
{Random-EI}     & 43.65 $\pm$ 1.21 & 48.55 $\pm$ 1.67 & 50.40 $\pm$ 1.33         \\ \hline
\end{tabular}
\end{adjustbox}
\begin{adjustbox}{width=\linewidth}
\begin{tabular}{c|l|l|l} 
  & \multicolumn{3}{c}{\textbf{\textit{Covertype}}} \\ \hline
{\textit{U }}              & 20\%                        & 40\%                        & 60\%             \\ \hline
TSRC              &  54.64 $\pm$ 1.87          &  63.35 $\pm$ 1.87           &  69.59 $\pm$ 1.72           \\ \hline
CATS              & 65.57 $\pm$ 2.17  &  \textbf{72.58 $\pm$ 2.36}   & 78.58 $\pm$ 2.35\\ \hline
CATS-fix       &  \textbf{65.88 $\pm$ 2.01}         &  72.67 $\pm$ 2.13          &  78.55 $\pm$ 2.25\\ \hline
CALINUCB          &  61.99 $\pm$ 1.53          &  72.54 $\pm$ 1.76    &   \textbf{79.69 $\pm$ 1.82}\\ \hline
{Random-fix}      & 53.11 $\pm$ 1.45 & 59.67 $\pm$ 1.07                     & 64.18 $\pm$ 1.03         \\ \hline
{Random-EI}       & 46.15 $\pm$ 2.61 & 52.55 $\pm$ 1.81                     & 55.45 $\pm$ 1.5         \\ \hline
\end{tabular}
\end{adjustbox}
\begin{adjustbox}{width=\linewidth}
\begin{tabular}{c|l|l|l}
                    & \multicolumn{3}{c}{\textbf{\textit{CNAE-9}}}                                                                                                                           \\ \hline
{\textit{U }}  & 20\%                        & 40\%                        & 60\%                                 \\ \hline
TSRC             & \textbf{33.57 $\pm$ 2.43} &       38.62 $\pm$ 1.68          & \textbf{42.05 $\pm$ 2.14} \\ \hline
CATS             & 29.84 $\pm$ 1.82          &  39.10 $\pm$ 1.41 &    40.52 $\pm$ 1.42 \\ \hline
CATS-fix         & 29.82 $\pm$ 1.70          & \textbf{39.57 $\pm$ 1.23} &    41.43 $\pm$ 1.39 \\ \hline
CALINUCB         & 28.53 $\pm$ 1.65          &         38.88 $\pm$ 1.35 &     39.73 $\pm$ 1.36 \\ \hline
{Random-fix}       & 33.01 $\pm$ 1.82 & 37.67 $\pm$ 1.68 & 39.18 $\pm$ 1.52         \\ \hline
{Random-EI}   & 32.05 $\pm$ 2.01 & 36.65 $\pm$ 1.90 & 37.47 $\pm$ 1.75         \\ \hline
\end{tabular}
\end{adjustbox}
\end{subtable}
\begin{subtable}[ht]{0.4\linewidth}\centering
\caption{Nonstationary setting}
\label{table:nonstationary}
\begin{adjustbox}{width=\linewidth}
\begin{tabular}{c|l|l|l}
                    & \multicolumn{3}{c}{\textbf{\textit{Warfarin}}}                                                                                                                       \\ \hline
\textit{U} & 20\%                        & 40\%                        & 60\%                     \\ \hline
WTSRC              & 55.83 $\pm$ 0.55            & 58.00 $\pm$ 0.83          & 59.85 $\pm$ 0.60          \\ \hline
NCATS              & \textbf{59.47 $\pm$ 2.89}   & \textbf{59.34 $\pm$ 2.04} & \textbf{63.26 $\pm$ 0.75} \\ \hline
NCATS-fix          &         59.01 $\pm$ 3.09 &  59.14 $\pm$ 2.33  &  62.42 $\pm$ 0.98 \\ \hline
NCLINUCB           &  58.64 $\pm$ 2.77           &  58.43 $\pm$ 1.89         &  63.01 $\pm$ 0.66  \\ \hline
{Random-fix}       & 43.91 $\pm$ 1.17            & 47.67 $\pm$ 1.08 & 54.18 $\pm$ 1.03         \\ \hline
{Random-EI}   & 47.78 $\pm$ 2.11                 & 52.55 $\pm$ 1.83 & 55.45 $\pm$ 1.54         \\ \hline
\end{tabular}
\end{adjustbox}
\begin{adjustbox}{width=\linewidth}
\begin{tabular}{c|l|l|l}
                    & \multicolumn{3}{c}{\textbf{\textit{Covertype}}   }                                                                                                                   \\ \hline
\textit{U}  & 20\%                        & 40\%                        & 60\%                        \\ \hline
WTSRC               & \textbf{50.26 $\pm$ 1.58} & 58.99 $\pm$ 1.81             & 64.91 $\pm$  1.38         \\ \hline
NCATS                  & 48.50 $\pm$ 1.05          & \textbf{68.17 $\pm$ 3.14} & \textbf{83.78 $\pm$ 5.51} \\ \hline
NCATS-fix              & 49.87 $\pm$ 1.20          & 68.04 $\pm$ 3.24       & 82.98 $\pm$ 5.83 \\ \hline
NCLINUCB               & 48.12 $\pm$ 0.99          &         68.20 $\pm$ 3.11  &  83.91 $\pm$ 5.21 \\ \hline
{Random-fix}       & 43.11 $\pm$ 3.05 & 49.67 $\pm$ 2.77 & 53.18 $\pm$ 2.33         \\ \hline
{Random-EI}   & 44.45 $\pm$ 4.44 & 46.65 $\pm$ 3.88 & 53.45 $\pm$ 3.61         \\ \hline
\end{tabular}
\end{adjustbox}

\begin{adjustbox}{width=\linewidth}
\begin{tabular}{c|l|l|l}
                    & \multicolumn{3}{c}{\textbf{\textit{CNAE-9}}}                                                                                                                           \\ \hline
\textit{U}  & 20\%                        & 40\%                        & 60\%                   \\ \hline
WTSRC               & 19.91 $\pm$ 2.67               & 30.86 $\pm$ 2.92          & 36.01 $\pm$ 2.88       \\ \hline
NCATS               &         30.88 $\pm$ 0.96       & \textbf{34.91 $\pm$ 1.93} & \textbf{42.04 $\pm$ 1.52} \\ \hline
NCATS-fix        &      29.92 $\pm$ 1.06          &        33.43 $\pm$ 1.83 &          40.04 $\pm$ 1.51 \\ \hline
NCLINUCB            &         \textbf{31.07 $\pm$ 0.87} &         34.61 $\pm$ 1.73  &         41.81 $\pm$ 1.62 \\ \hline
{Random-fix}        & 13.01 $\pm$ 3.45 & 21.77 $\pm$ 3.08 & 24.18 $\pm$ 2.43         \\ \hline
{Random-EI}   & 16.15 $\pm$ 2.44 & 22.55 $\pm$ 2.18 & 25.45 $\pm$ 2.15         \\ \hline
\end{tabular}
\end{adjustbox}
\end{subtable}
\end{table*}
The results in Table \ref{tab:uci} are promising, with {\em our methods outperforming the state of the art in the majority of cases across both settings}. The most notable exception is found for CNAE-9 dataset, where CATS sometimes outperforms or nearly matches TSRC performance. This outcome is somewhat expected, since in the original work on TSRC \cite{BouneffoufRCF17}, the mean error rate of TSRC was only $0.03\%$ lower than the error corresponding to randomly fixing a subset of unknown features to reveal for each event on CNAE-9. This observation suggests that, for this particular dataset, there may not be a subset of features which would be noticeably more predictive of the reward than the rest of the features. We also observe that the LINUCB version of CATS has comparable performance with CATS with slight advantage to CATS. Another observation is that CATS-fix is performing better than CATS in some situations, the explanation could be that after finding the best features the algorithm do not need to explore anymore and focus on finding the best arms based on these featues.


\begin{figure}[ht]
\vspace{-0.15in}
\centering
\begin{subfigure}{0.350\linewidth}
  \centering
  \includegraphics[width=\linewidth]{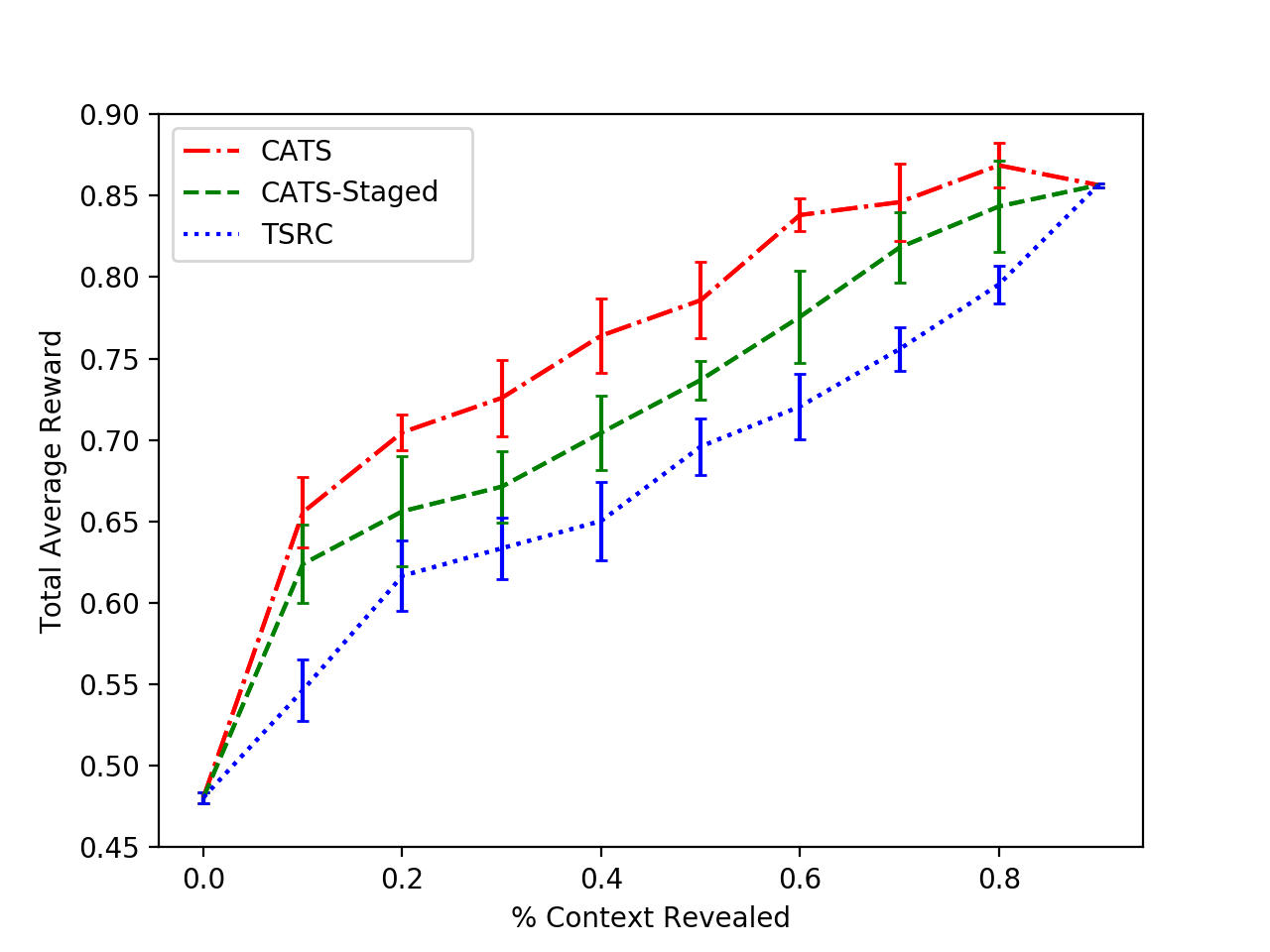}
  \caption{Stationary Setting}
  \label{fig:covS}
\end{subfigure}%
\begin{subfigure}{0.350\linewidth}
  \centering
  \includegraphics[width=\linewidth]{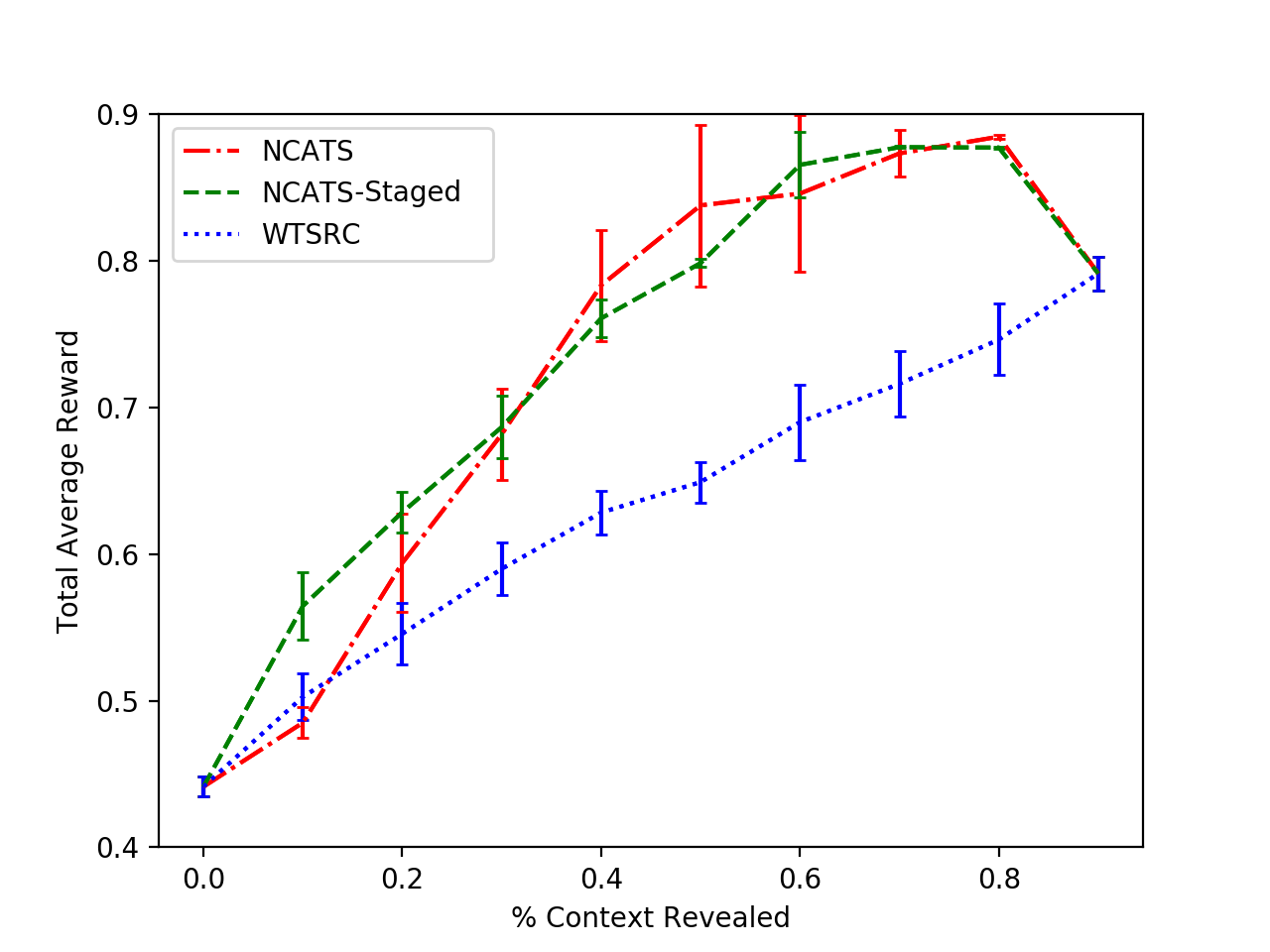}
  \caption{Nonstationary Setting}
  \label{fig:covN}
\end{subfigure}
\caption{Total Average Reward for Covertype}
\end{figure}

We perform a deeper analysis of the Covertype dataset, examining multi-staged selection of the $U$ unknown context feature sets. In CATS, the known context is used to select all $U$ additional context feature sets at once. In a multi-staged approach, the known context grows and is used to select each of the $U$ additional context features incrementally (one feature at a time). Maintaining $\lambda(t)=1$, for the stationary case we denote these two cases of the CATS algorithm as CATS and CATS-Staged respectively  and report their performance when $10\%$ of the context is randomly fixed, across various $U$ in Figure \ref{fig:covS}. Note that when the remaining 90\% of features are revealed, the CATS and TSRC methods all reduce to simple linear Thompson sampling  with the full feature set. Similarly, when 0 additional feature sets are revealed, the methods all reduce to linear Thompson sampling with a sparsely represented known context. Observe that CATS consistently outperforms CATS-Staged across all $U$ tested. CATS-Staged likely suffers because incremental feature selection adds nonstationarity to the known context - CATS  learns relationships between the known and unknown features while CATS-Staged learns relationships between them as the known context grows. Nonetheless, both methods outperform TSRC. In the nonstationary case we use the GP-UCB algorithm for $\lambda(t)$, refer to the single and multi-staged cases as NCATS and NCATS-Staged, and illustrate their performance in Figure \ref{fig:covN}. Here we observe that NCATS and NCATS-Staged have comparable performance, and the improvement gain over baseline, in this case WTSRC, is even greater than in the stationary case. 

Next we evaluate our methods on \textit{Customer Assistant}, a proprietary multi-skill dialog orchestration dataset. Recall that this kind of application motivates the CAB setting because there is a natural divide between the known and unknown context spaces; the query and its associated features are known at the onset and the potential responses and their associated features are only known for the domain specific agents the query is posed to. 
\textit{The Customer Assistant} orchestrates 9 domain specific agents which we arbitrarily denote as $Skill_{1}, \ldots,  Skill_{9}$ in the discussion that follows. In this application, example skills lie in the domains of payroll, compensation, travel, health benefits, and so on. In addition to a textual response to a user query, the skills orchestrated by \textit{Customer Assistant} also return the following features: an \textit{intent}, a short string descriptor that categorizes the perceived intent of the query, and a \textit{confidence}, a real value between 0 and 1 indicating how confident a skill is that its response is relevant to the query. Skills have multiple intents associated with them. The orchestrator uses all the features associated with the query and the candidate responses from all the skills to choose which skill should carry the conversation.

The Customer Assistant dataset contains 28,412 events associated with a correct skill response. We encode each query by averaging 50 dimensional GloVe word embeddings \cite{pennington2014glove} for each word in each query and for each skill we create a feature set consisting of its confidence and a one-hot encoding of its intent. The skill feature set size for $Skill_{1}, \ldots,  Skill_{9}$ are 181, 9, 4, 7, 6, 27, 110, 297, and 30 respectively. We concatenate the query features and all of the skill features to form a 721 dimensional context feature vector for each event in this dataset. Recall that there is no need for simulation of the known and unknown contexts; in a live setting the query features are immediately calculable or known, whereas the confidence and intent necessary to build a skill's feature set are unknown until a skill is executed. Because the confidence and intent for a skill are both accessible post execution, we reveal them together. We accommodate this by slightly modifying the objective of CATS to reveal $U$ unknown skill feature sets instead of $U$ unknown individual features for each event. 
We perform a deeper analysis of the \textit{Customer Assistant} dataset, examining multi-staged selection of the $U$ unknown context feature sets. Maintaining $\lambda(t)=1$, for the stationary case the results are summarized in Figure \ref{fig:chipS}. Here both CATS-Staged and CATS methods outperform TSRC by a large margin. 

\begin{figure}[ht]
\centering
\begin{subfigure}{0.35\linewidth}
  \centering
  \includegraphics[width=\linewidth]{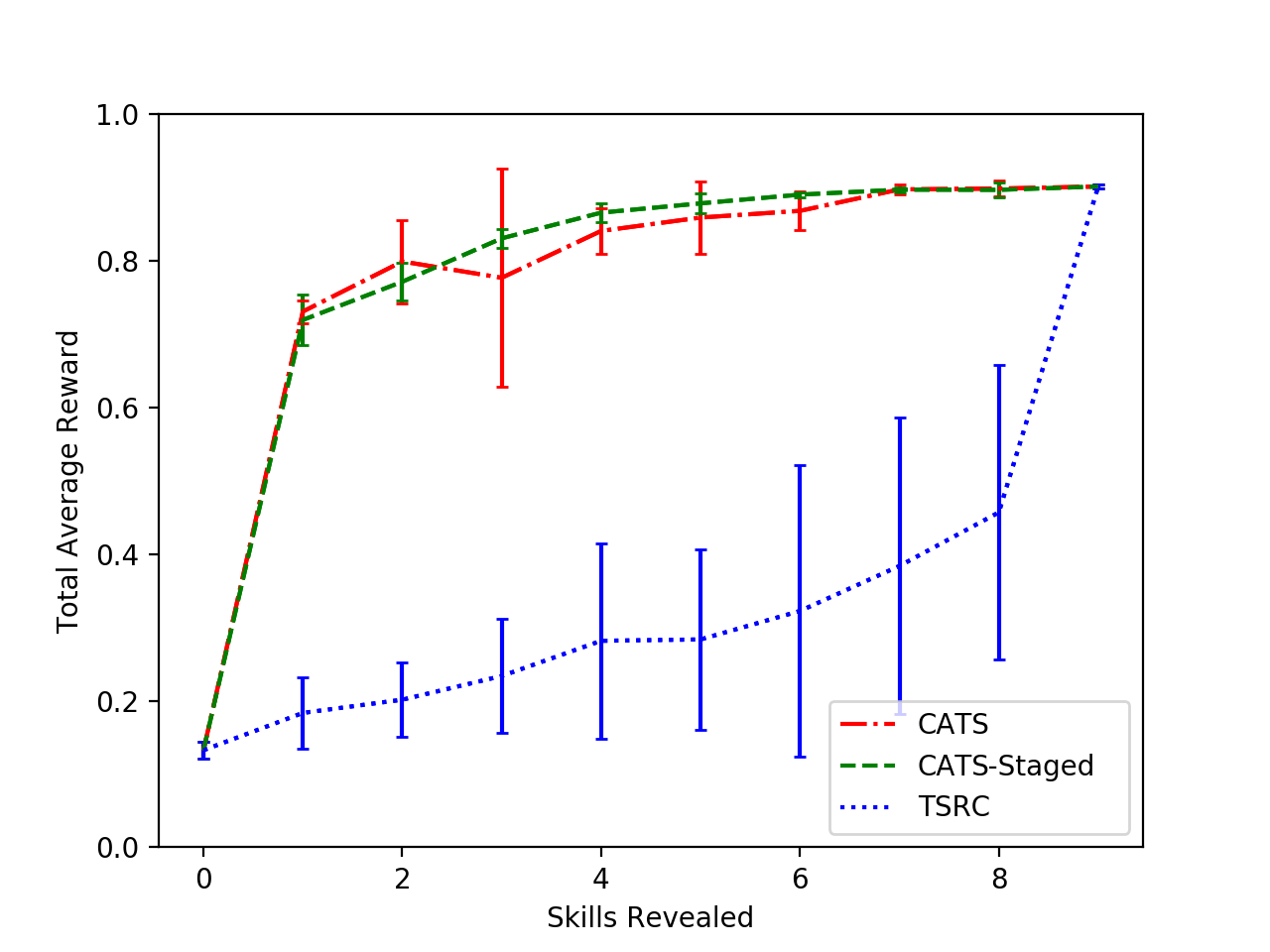}
  \caption{Stationary Setting}
  \label{fig:chipS}
\end{subfigure}%
\begin{subfigure}{0.35\linewidth}
  \centering
  \includegraphics[width=\linewidth]{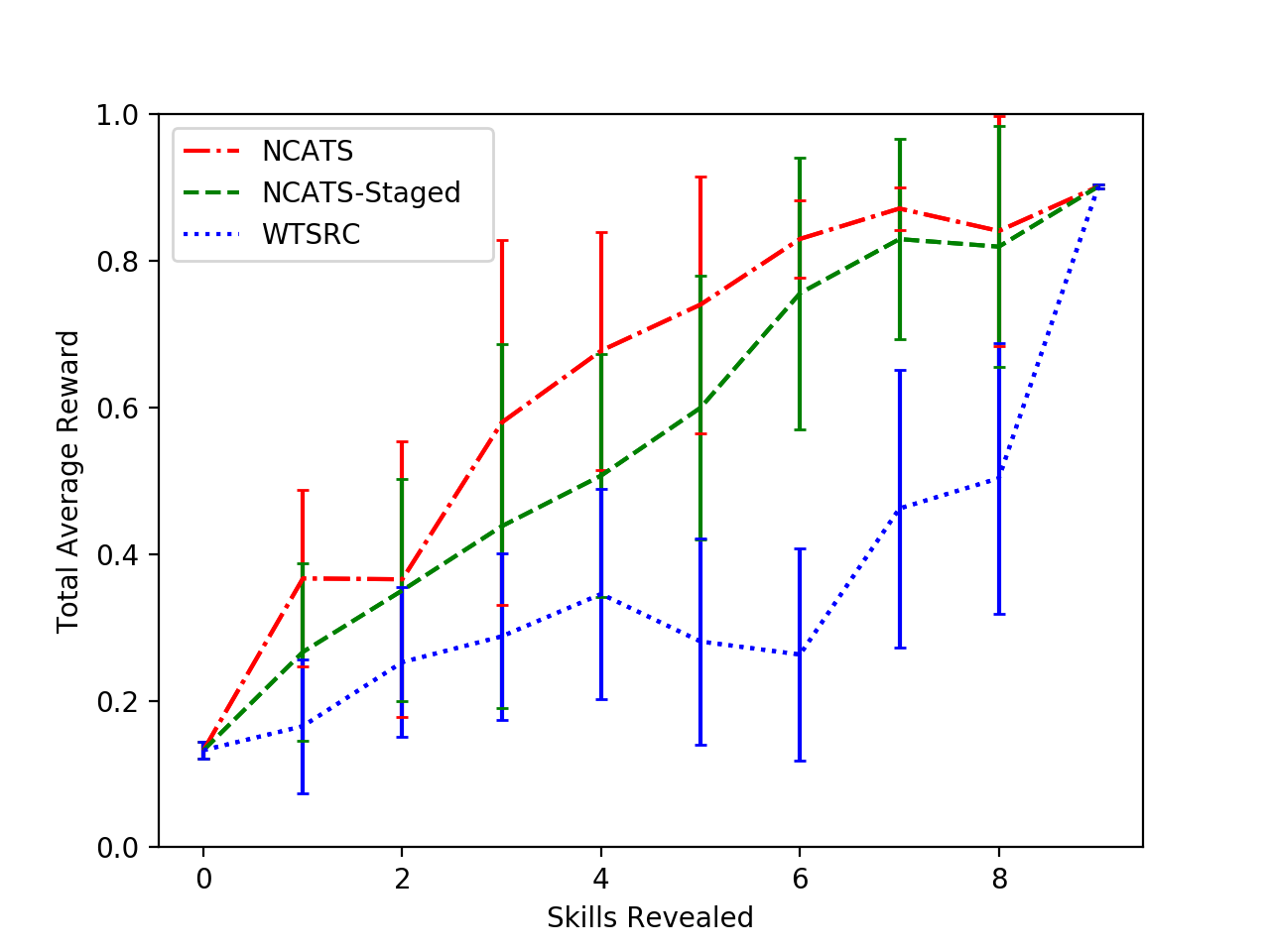}
  \caption{Nonstationary Setting}
  \label{fig:chipN}
\end{subfigure}
\caption{Total Average Reward for Customer Assistant}
\end{figure}
For the nonstationary case we simulate nonstationarity in the same manner as the publicly available datasets, except using the natural partition of the query features as the known context and the skill feature sets as the unknown context instead of simulated percentages. We use the GP-UCB algorithm for $\lambda(t)$ and illustrate the performance of NCATS and NCATS-Staged alongside WTSRC in Figure \ref{fig:chipN}. Here we observe that NCATS slightly outperforms NCATS-Staged, and both outperform the WTSRC baseline.  

\section{Conclusions and Future Work}
We have introduced here a novel bandit problem with only partially observable context and the option of requesting a limited number of additional observations. We also propose an algorithm, designed to take an advantage of the initial partial observations in order to improve its choice of which additional features to observe, and demonstrate its advantages over the prior art, a standard context-attentive bandit with no partial observations of the context prior to feature selection step. Our problem setting is motivated by several realistic  scenarios, including medical applications as well as multi-domain dialog systems. 
Note that our current formulation assumes that all unobserved features have equal observation cost. However, a more practical assumption is that some features may be more costly than others; thus, in our future work, we plan to expand this notion of budget to accommodate more scenarios involving different feature costs.  

\section{Broader Impact}
This problem has broader impacts in several domains such as voice assistants, healthcare and e-commerce.
\begin{itemize} 

\item \textit{Better medical diagnosis.} In a clinical setting, it is often too costly or infeasible to conduct all possible tests; therefore, given the limit on the number of tests, the doctor must decide which subset of tests will result into maximally effective treatment choice in an iterative manner. A doctor may first take a look at patient's medical record to decide which medical test to perform, before choosing a treatment plan. 

\item \textit{Better user preference modeling.} Our approach can help to develop better chatbots and automated personal assistants. For example, following a request such as, for example, "play music", an AI-based home assistant must learn to ask several follow-up questions (from a list of possible questions) to better understand the intent of a user and to remove ambiguities: e.g., what type of music do you prefer (jazz, pop, etc)? Would you like it on hi-fi system or on TV? And so on. Another example: a support desk chatbot, in response to user’s complaint ("My Internet connection is bad") must learn to ask a sequence of appropriate questions (from a list of possible connection issues): how far is your WIFI hotspot? Do you have a 4G subscription? These scenarios are well-handled by the framework we proposed in this paper.

\item \textit{Better recommendations.} Voice assistants and recommendation systems in general tend to lock us in our preferences, which can have deleterious effects: e.g.,  recommendations  based only on the past history of user's choices may reinforce certain undesirable tendencies, e.g., suggesting an online content based  on a user's with particular bias (e.g.,  racist, sexist, etc). On the contrary,  our approach could potentially help a user   to break out of this loop, by  suggesting the items (e.g. news)  on  additional questions (additional features) which can be used to broaden user's horizons.
\end{itemize} 

\bibliographystyle{named}
\interlinepenalty=10000
\bibliography{neurips20}

\end{document}